%% file: paper.tex
\documentclass[10pt,draftcls,onecolumn]{IEEEtran}

\usepackage{microtype}

\usepackage{booktabs} % for professional tables

\usepackage{mathrsfs}
\usepackage[font={small}]{caption}
\usepackage{scalerel}
\usepackage{url}
\usepackage{bbold}
\usepackage{amsfonts}
\usepackage{amsmath,amssymb}
\usepackage{graphicx}
\usepackage{hyperref}
\usepackage{wrapfig}
\usepackage{mathrsfs} 
\usepackage{algorithm,algorithmic}
\usepackage{array}
\usepackage{times}
\usepackage{url}
\usepackage{subfigure}
\usepackage{cite}
\usepackage{upgreek}
\usepackage{amsthm}
\usepackage{float}
\usepackage{longtable}
\usepackage{color}

\begin{document}

\title{\bf \LARGE On Sampling Random Features From Empirical Leverage Scores: Implementation and Theoretical Guarantees}

\author{Shahin Shahrampour$^1$ and Soheil Kolouri$^2$\footnote{

\noindent 
$[1]$ Shahin Shahrampour is with Texas A\&M University, College Station, TX, USA (e-mail: shahin@tamu.edu).\\
$[2]$ Soheil Kolouri is with HRL Laboratories, LLC., Malibu, CA, USA (e-mail: skolouri@hrl.com).}} %

\input{header}

\maketitle

\begin{abstract}
Random features provide a practical framework for large-scale kernel approximation and supervised learning. It has been shown that {\it data-dependent} sampling of random features using leverage scores can significantly reduce the number of features required to achieve optimal learning bounds. Leverage scores introduce an optimized distribution for features based on an infinite-dimensional integral operator (depending on input distribution), which is impractical to sample from. Focusing on empirical leverage scores in this paper, we establish an out-of-sample performance bound, revealing an interesting trade-off between the approximated kernel and the eigenvalue decay of another kernel in the domain of random features defined based on data distribution. Our experiments verify that the empirical algorithm consistently outperforms vanilla Monte Carlo sampling, and with a minor modification the method is even competitive to supervised data-dependent kernel learning, without using the output (label) information. 
\end{abstract}

\section{Introduction}\label{intro}
Supervised learning is a fundamental machine learning problem, where a learner is given input-output data samples (from an unknown distribution), and the objective is to find a mapping from inputs to outputs \cite{friedman2001elements}. Kernel methods are powerful tools to capture the nonlinear relationship between input-outputs. These methods {\it implicitly} map the inputs (features) to a high-dimensional space, without the need for knowledge of the feature map, an idea known as kernel trick. While kernel methods are theoretically well-justified,
their practical applicability to large datasets is limited in that they require memory (and time) complexity that can scale quadratically (and cubically) with the size of data samples. 

In the past few years, this computational bottleneck has motivated a large body of research on (low-rank) kernel approximation   \cite{smola2000sparse,fine2001efficient,rahimi2007random} for efficient learning. In these scenarios, the training can scale linearly with respect to data, introducing a dramatic decrease in the computational cost. In this line of work, an elegant idea has been the use of the so-called {\it random features} for kernel approximation \cite{rahimi2007random} as well as training shallow networks \cite{rahimi2009weighted}. In this approach, random features are sampled from a stochastic oracle to form the nonlinear basis functions used to describe the input-output relationship. Replacing optimization, randomization circumvents the non-convexity in training and comes with a theoretical generalization guarantee \cite{rahimi2009weighted}.

Since its development, the randomized-feature approach has been successfully used for a wide range of problems (see e.g. \cite{si2016goal} for matrix completion, \cite{lopez2013randomized} for the correlation analysis of random variables, and \cite{carratino2018learning} for non-parametric statistical learning), but as pointed out in \cite{yang2012nystrom}, since the basis functions are sampled from a distribution that is {\it independent} of data, the number of features required to learn the data subspace may be large. Therefore, a natural question is whether a {\em data-dependent} stochastic oracle can prove to be useful in improving the out-of-sample performance.

Recently, a number of works have developed {\it supervised} data-dependent methods for sampling random features with the goal of improving generalization \cite{sinha2016learning,AAAI1816684,bullins2017not}. This objective is achieved by pre-processing the random features (e.g. via optimizing a metric) and focusing on promising features, which amounts to learning a ``good'' kernel based on {\it input-output} pairs. We provide a comprehensive review of these works in Section \ref{RL}, but the focus of this work is on an {\it unsupervised} data-dependent method relying on leverage scores, calculated based only on inputs \cite{bach2017equivalence,rudi2016generalization,sun2018but}. \cite{rudi2016generalization} have discussed the impact of leverage scores for ridge regression, \cite{sun2018but} addressed the problem in the case of SVM, and \cite{bach2017equivalence} has established theoretical guarantees for Lipschitz continuous loss functions. Common in all these results is the fact that using leverage scores for sampling random features can significantly reduce the number of features required to achieve optimal learning bounds. The bounds are particularly useful when the eigenvalues of the integral operator corresponding to the underlying kernel decay fast enough. Nevertheless, these works do not aim to change the underlying base kernel. 

There are two {\it practical hurdles} in using leverage scores: (i) they introduce an optimized distribution for re-sampling features, which is based on the infinite-dimensional integral operator associated to the underlying kernel, and (ii) the support set (domain of random features) is infinite-dimensional, making the optimized distribution impractical to sample from. An empirical sampling scheme is proposed in the experiments of \cite{sun2018but} without the theoretical analysis, and as noted in \cite{sun2018but} a result in the theoretical direction will be useful for guiding practitioners.

In this paper, we aim to address the problem above using {\it empirical leverage scores}. In this scenario, we must construct a {\it finite} counterpart of the optimized distribution to use for training. Interestingly, the out-of-sample performance of the algorithm (Theorem \ref{thm}) reveals an interesting trade-off between two errors: (i) the approximation error of the kernel caused by finiteness of random features, and (ii) the eigenvalue decay of another kernel in the domain of random features defined based on data distribution. The proof of our main result uses a combination of the approximation error result of \cite{bach2017equivalence}
as well as the spectral approximation result of \cite{avron2017random} for ridge leverage functions (which builds on recent works on matrix concentration inequalities \cite{tropp2015introduction}). We also verify with numerical experiments (on practical datasets) that the empirical leverage score idea consistently outperforms vanilla Monte Carlo sampling \cite{rahimi2009weighted}, and with a minor modification in the sampling scheme, it can be even competitive to {\it supervised} data-dependent methods, {\it without using outputs (labels).}

%%%%%%%%%%%%%%%%%%      Problem

\section{Problem Formulation}
\paragraph {Preliminaries:} We denote by $[N]$ the set of positive integers $\{1,\ldots,N\}$, by $\tr{\cdot}$ the trace operator, by $\norm{\cdot}$ the spectral (respectively, Euclidean) norm of a matrix (respectively, vector), by $\ex{\cdot}$ the expectation operator, and by $\text{var}(\cdot)$ the variance operator. Boldface lowercase variables (e.g. $\ab$) are used for vectors, and boldface uppercase variables (e.g. $\Ab$) are used for matrices. $[\Ab]_{ij}$ denotes the $ij$-th entry of matrix $\Ab$. The vectors are all in column form.

$\Lc^2(dp,\Xc)$ represents the set of square integrable functions with respect the Borel probability
measure $dp$ on the domain $\Xc$. We use $\inn{\cdot,\cdot}_{\Fc}$ to denote the inner product associated to an inner product space $\Fc$ and $\norm{\cdot}_{\Fc}$ for its corresponding norm. The subscript may be dropped when it is clear from the context (e.g. for the Euclidean space). For a positive semi-definite linear operator $\Sigma$, the sequence $\{\sigma_i(\Sigma)\}_{i=1}^\infty$ denotes the set of (non-negative) eigenvalues in descending order. The sequence is finite if $\Sigma$ is finite-dimensional.

\subsection{Supervised Learning}\label{framework}
In the supervised learning problem, we are given a training set $\{(\xb_n,y_n)\}_{n=1}^N$ in the form of {\it input-output} pairs, which are i.i.d. samples from an {\it unknown} distribution. The input feature space is $d$-dimensional, i.e., for $n\in [N]$, we have $\xb_n \in \Xc \subset \R^d$, where $\Xc$ is closed and convex. For regression, we assume $y_n \in \Yc \subseteq [-1,1]$, whereas for classification we have $y_n \in \{-1,1\}$. The goal of supervised learning is to find a function $f: \Xc  \to \R$ based on the training set, which can generalize well, i.e., it can accurately predict the outputs of previously unseen inputs.

The problem above can be formulated as minimizing a risk functional $R(f)$, defined as
\begin{align*}
R(f)\triangleq\E[L(f(\xb),y)] ~~~~~~~~~~ \Rh(f)\triangleq\frac{1}{N}\sum\limits_{n=1}^N L(f(\xb_n),y_n),
\end{align*}
where $L(\cdot,\cdot)$ is a task-dependent loss function (e.g. hinge loss for SVM), and the expectation is taken with respect to data. As this distribution is unknown, we can only minimize the empirical risk $\Rh(f)$, instead of the true risk $R(f)$, and calculate the gap between the two using standard arguments from measures of function space complexity (e.g. VC dimension, Rademacher complexity, etc). We will discuss two related function classes in the next section.

\subsection{Kernels and Random Features}
To minimize the risk functional, we need to focus on a function class for $f(\cdot)$. Let us consider a symmetric positive-definite function $k(\cdot,\cdot)$ such that $\sum^N_{i,j=1}\alpha_i \alpha_j k(\xb_i,\xb_j)\geq 0$ for $\alphab\in \R^N$. $k(\cdot,\cdot)$ is then called a positive (semi-)definite kernel, and a possible class to consider is the Reproducing Kernel Hilbert Space (RKHS) associated to $k(\cdot,\cdot)$, defined as follows
\begin{align}\label{kernelclass}
    \Fc_k\triangleq\left\{f(\cdot)=\sum\limits_{n=1}^N \alpha_nk(\xb_n,\cdot): \alphab \in \R^N\right\}.
\end{align}
Minimizing the empirical risk $\Rh(f)$ over this class of functions by optimizing over $\alphab$ is theoretically well-understood and justified; however, since this approach requires $O(N^2)$ in space and $O(N^3)$ in time (e.g. training ridge regression with naive matrix inversion), the practical applicability of kernel methods to large datasets is limited.

It is often useful to study RKHS through the following integral operator $\Sigma:\Lc^2(dp,\Xc) \to \Lc^2(dp,\Xc)$
\begin{align}\label{operator}
    (\Sigma f)(\cdot)=\int_\Xc f(\xb)k(\xb,\cdot)dp(\xb).
\end{align}
The spectral properties of the kernel matrix $[\Kb]_{ij}=k(\xb_i,\xb_j)/N$ is related to that of $\Sigma$ (see e.g. \cite{braun2005spectral}). When $\sup_{\xb\in\Xc}k(\xb,\xb)<\infty$, $\Sigma$ is self-adjoint, positive semi-definite and trace-class\footnote{Note that this is a sufficient (but not a necessary) condition. $\int_\Xc k(\xb,\xb)dp(\xb)<\infty$ is a weaker condition for which we can have the same properties \cite{bach2017equivalence}.}.

Let us now restrict our attention to kernels that can be written as, 
\begin{align}\label{integralform}
    k(\xb,\xb')=\int_\Omega \phi(\xb ,\omegab)\phi(\xb',\omegab)d\tau(\omegab),
\end{align}
for all $\xb, \xb' \in \Xc$ and a measure $d\tau$ on $\Omega$. Many common kernels can take the form above. Examples include shift-invariant kernels  \cite{rahimi2007random} or dot product (e.g. polynomial) kernels \cite{kar2012random}\footnote{We refer the reader to Table 1 in \cite{yang2014random} as well as Table 1 in \cite{liao2018spectrum} for an exhaustive list.}.

The integral form \eqref{integralform} can be approximated using Monte Carlo sampling of so-called {\it random features} $\{\omegab_m\}_{m=1}^M$, which are i.i.d. vectors generated from $d\tau$. Then, 
\begin{align*}
\kh_{M}(\xb,\xb') &\triangleq \frac{1}{M}\sum\limits_{m=1}^M\phi(\xb ,\omegab_m)\phi(\xb',\omegab_m)=\inn{\phib_M(\xb),\phib_M(\xb')}, \numberthis\label{MC}
\end{align*}
where
\begin{align}\label{PhiM}
    \phib_M(\xb)\triangleq\frac{1}{\sqrt{M}}[\phi(\xb,\omegab_1),\ldots,\phi(\xb,\omegab_M)]^\top,
\end{align}
naturally leading to the function class
\begin{align}\label{RFclass}
\Fch \triangleq \left\{ f(\cdot)=\sum\limits_{m=1}^M  \theta_{m}\phi(\cdot,\omegab_m): \thetab \in \R^M\right\}.
\end{align}
The advantage of optimizing the risk function on $\Fch$ (rather than $\Fc$) is that the training can be considerably more efficient if we can keep $M\ll N$. For example, in the case of ridge regression, the $O(N^3)$ time would reduce to $O(M^3+NM^2)$.   

In fact, recently \cite{rudi2016generalization} showed that to achieve the same statistical accuracy as kernel ridge regression (i.e., $O(1/\sqrt{N})$ risk error), we only require $M=O(\sqrt{N}\log N)$ random features using vanilla Monte Carlo sampling. Note that randomized-feature approach would also reduce the computation time of the test phase from $O(N)$ to $O(M)$.

\subsection{Leverage Scores and Data-Dependent Sampling}
The function class \eqref{RFclass} can be also viewed as a one-(hidden)layer neural network (i.e., a perceptron) with an activation function $\phi(\cdot,\cdot)$. To minimize the empirical risk over \eqref{RFclass}, we can (in general) consider three possible paths: (1) Joint optimization over $\thetab$ and $\{\omegab_i\}_{i=1}^M$, which fully trains the neural network by solving a non-convex optimization. (2) Monte Carlo sampling of $\{\omegab_i\}_{i=1}^M$ and optimizing over $\thetab$ \cite{rahimi2009weighted}, which was discussed in the previous section. (3) Data-dependent sampling of $\{\omegab_i\}_{i=1}^M$ and optimizing over $\thetab$. Though (1) seems to be the most powerful technique, the main advantage of (2) and (3) is dealing with a convex problem that avoids (potentially bad) local minima. Another potential advantage is that we do not require the gradient of the activation function for training, which broadens the scope of applicability. 

Recently, a number of works have proposed {\it supervised} data-dependent sampling of random features to enhance the generalization \cite{sinha2016learning,AAAI1816684,bullins2017not}. This objective is achieved by pre-processing the random features (e.g. via optimizing a metric) and focusing on ``good'' ones (for the generalization purpose). We provide a comprehensive review of these works in Section \ref{RL}, and here, we focus on presenting a promising {\it unsupervised} data-dependent method that relies upon leverage scores \cite{bach2017equivalence,rudi2016generalization,sun2018but}.

\cite{rudi2016generalization,sun2018but} have discussed the impact of leverage scores for ridge regression and SVM. For $\omegab\in\Omega$, leverage score is defined as \cite{bach2017equivalence}
\begin{align}\label{score}
 s(\omegab)\triangleq \inn{\phi(\cdot,\omegab),(\Sigma+\lambda\Ib)^{-1}\phi(\cdot,\omegab)}_{\Lc^2(dp,\Xc)},  
\end{align}
where $\Sigma$ is the integral operator in \eqref{operator}. In turn, the optimized distribution for random features is derived as follows
\begin{align}\label{optimaldist}
    q(\omegab)=\frac{s(\omegab)}{\int_\Omega s(\omegab)d\tau(\omegab)}.
\end{align}
Notice that if we have access to $q(\omegab)$, the unbiased approximation of the kernel takes the form
\begin{align}\label{kernelq}
k(\xb,\xb') &\approx \frac{1}{M}\sum\limits_{m=1}^M\frac{1}{q(\omegab_m)}\phi(\xb ,\omegab_m)\phi(\xb',\omegab_m),
\end{align}
with respect to the new measure $q(\omegab)d\tau(\omegab)$.
All of the aforementioned works have established theoretical results, showing that if the eigenvalues of $\Sigma$ decay fast enough, the number of random features to achieve $O(1/\sqrt{N})$ error can significantly decrease ( to $\log(N)$ and even constant!). There are, however, practical challenges to consider. 

{\bf Practical challenges:} We can observe that sampling random features according to $s(\omegab)$ gives rise to two issues \cite{bach2017equivalence}: (i) 
we require the knowledge of the infinite-dimensional operator $\Sigma$ (which is not available), and (ii) the set $\Omega$ might be large and impractical to sample from. An empirical mechanism of sampling has been proposed in the experiments of \cite{sun2018but} without the theoretical analysis. As noted in \cite{sun2018but} a result in the theoretical direction will be extremely useful for guiding practitioners, and we will discuss that in Section \ref{theory} after outlining the empirical leverage scores next.

\subsection{Sampling Based on Empirical Leverage Scores}\label{Problem4}
To start, let us first define the matrix
\begin{align}\label{Phi}
    \Phib_{M,N}\triangleq\frac{1}{\sqrt{N}}[\phib_M(\xb_1),\ldots,\phib_M(\xb_N)]\in \R^{M\times N},
\end{align}
where $\phib_M(\cdot)$ is given in \eqref{PhiM}. Observe that $\Phib_{M,N}$ can be related to kernel function $k(\cdot,\cdot)$ as follows,
\begin{align}\label{Kb}
\Kb \triangleq \frac{1}{N}  \begin{bmatrix}
  k(\xb_1,\xb_1) & \cdots &   k(\xb_1,\xb_N)\\
  \vdots &   \vdots &   \vdots  \\ 
    k(\xb_N,\xb_1) & \cdots & k(\xb_N,\xb_N)
   \end{bmatrix}=\E_{d\tau}\left[\Phib_{M,N}^\top\Phib_{M,N}\right] 
\end{align}
Now, consider another kernel $g: \Omega \times \Omega \to \R$ defined as $g(\omegab,\omegab')=\int_{\Xc}\phi(\xb,\omegab)\phi(\xb,\omegab')dp(\xb)$, which measures the dissimilarity of random features. Then, the following relationship holds
\begin{align}\label{Gb}
\Gb \triangleq \frac{1}{M}  \begin{bmatrix}
  g(\omegab_1,\omegab_1) & \cdots &   g(\omegab_1,\omegab_M)\\
  \vdots &   \vdots &   \vdots  \\ 
    g(\omegab_M,\omegab_1) & \cdots & g(\omegab_M,\omegab_M)
   \end{bmatrix}=\E_{dp}\left[\Phib_{M,N}\Phib_{M,N}^\top\right]. 
\end{align}
It is shown in \cite{bach2017equivalence} that sampling random features using leverage scores \eqref{score} corresponds to selecting (re-weighting) them according to the diagonal elements of the matrix 
\begin{align}\label{trueridge}
    \tilde{\Gb}\triangleq \Gb(\Gb+\lambda\Ib)^{-1}.
\end{align} 
Though the dependence to the operator $\Sigma$ is relaxed, still the dimension of $\Gb$ grows with the number of random features, which suggests that the more features we evaluate (from the set $\Omega$), the more computational cost we incur. More importantly, another hurdle is that the data distribution is unknown and $\Gb$ cannot be calculated. Therefore, appealing to $\Phib_{M,N}\Phib_{M,N}^\top$ seems to be a natural solution in practice. The outline of such method is given in Algorithm \ref{ALGO1}\footnote{This algorithm was also suggested in the experiments of \cite{sun2018but}}.
\begin{algorithm}[h!]
	\caption{Empirical Leverage Score Sampling (\alg)}
	{\bf Input:} 
	A sub-sample $\{\xb_n\}_{n=1}^{N_0}$ of inputs, the feature map $\phi(\cdot,\cdot)$, an integer $M_0$, the sampling distribution $d\tau$, the parameter $\lambda$.
	\begin{algorithmic}[1]\label{ALGO}
	\STATE Draw $M_0$ i.i.d. samples $\{\tilde{\omegab}_m\}_{m=1}^{M_0}$ according to $d\tau$.
	\STATE Construct the matrix 
	\begin{align}\label{Q}
	\Qb=\Phib_{M_0,N_0}\Phib_{M_0,N_0}^\top\left(\Phib_{M_0,N_0}\Phib_{M_0,N_0}^\top+\lambda\Ib\right)^{-1},
	\end{align}
	where $\Phib_{M_0,N_0}$ is defined in \eqref{Phi}.
	\STATE Let for $i\in [M_0]$
	\begin{align}
	    \qh(\tilde{\omegab}_i)=\frac{[\Qb]_{ii}}{\tr{\Qb}}.
	\end{align}
	\end{algorithmic}
	{\bf Output:} 
	The new weights $\qbh=[\qh(\tilde{\omegab}_1),\ldots,\qh(\tilde{\omegab}_{M_0})]^\top$.
\label{ALGO1}
\end{algorithm}

After running~\alg, we can use $\qbh$ as a discrete probability distribution to draw $M \leq M_0$ samples $\{\omegab_m\}_{m=1}^{M}$ and minimize the empirical risk over the function class
\begin{align}\label{WRFclass}
\Fch_{\qbh} \triangleq \left\{ f(\cdot)=\sum\limits_{m=1}^M  \theta_{m}\frac{\phi(\cdot,\omegab_m)}{\sqrt{\qh(\omegab_m)}}: \thetab \in \R^M\right\}.
\end{align}
The function class $\Fch_{\qbh}$ is defined in consistent with the choice of approximated kernel given in \eqref{kernelq}. Note that (assuming that we can calculate the inverse in \eqref{Q}) \alg~with $\lambda=0$ precisely recovers the Random Kitchen Sinks (RKS) \cite{rahimi2009weighted} and corresponds to uniform sampling. Figure \ref{fig:histograms} represents the histogram of weights with $M_0=2000$ features for the Year Prediction dataset. As expected, for $\lambda=1e-4$ the measure is uniform on all samples (bottom left) which translates to a delta plot for the histogram (top left). For $\lambda=1e+4$ (right) the empirical leverage scores transform the distribution of weights to a completely non-uniform measure.

The algorithm requires $O(N_0M_0^2+M_0^3)$ computations to form the matrix $\Qb$ in \eqref{Q} and calculate the empirical leverage scores (assuming naive inversion of matrix). Parameters $N_0$ and $M_0$ can be selected using rule-of-thumb (without exhaustive tuning). We elaborate on this in the numerical experiments (Section \ref{simulations}). Furthermore, the choice of the initial distribution $d\tau$ and the feature map $\phi(\cdot,\cdot)$ depend on the kernel that we want to use for training. For instance, cosine feature maps and Gaussian distribution can be used for approximating a Gaussian kernel \cite{rahimi2007random}.

\begin{remark}\label{R1}(Column Sampling)
 To improve efficiency, column sampling ideas have been previously used in the approximation of large kernel matrices \cite{bach2013sharp,rudi2018fast} in order to deal with (the ridge-type) matrix $\tilde{\Gb}$ in \eqref{trueridge}; however, those approaches are useful when the closed-form of the kernel matrix $\Gb$ is readily available, which is not the case in our setup, as the data distribution is unknown, and the kernel function $g(\cdot,\cdot)$ (in the domain of random features) must be approximated, i.e, we need to deal with \eqref{Q}.  
\end{remark}

\begin{remark}\label{R2}(Block Diagonal Approximation) Following Remark \ref{R1}, another technique to improve the time cost in \eqref{trueridge} is block kernel approximation. It has been shown in \cite{si2017memory} that for shift-invariant kernels, block kernel approximation can improve the approximation error (depending on the kernel parameter). For our setup, we still have the same problem as in Remark \ref{R1} (unknown $\tilde{\Gb}$).
\end{remark}

\begin{figure}[t]
    \centering
    \includegraphics[width=.5\columnwidth]{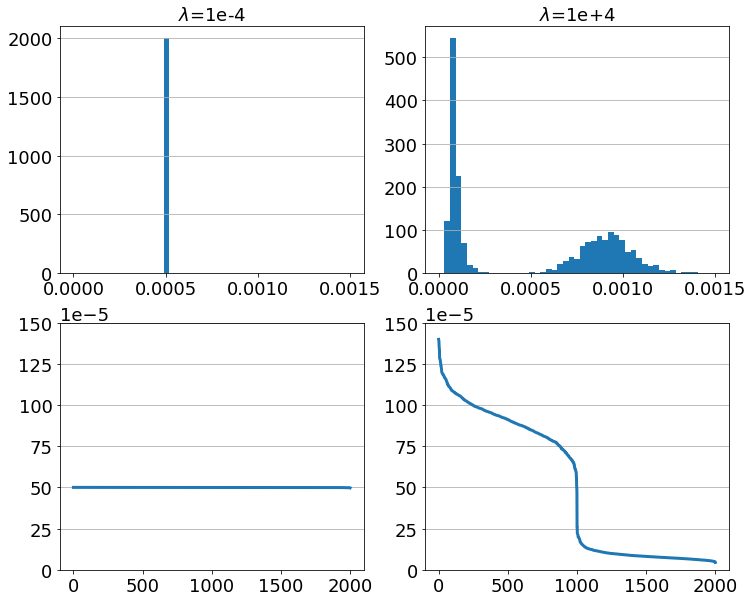}
    \caption{The histogram of weights calculated for $\lambda=1e-4$ (top left) and $\lambda=1e+4$ (top right) on the Year Prediction dataset, and the corresponding probability densities (bottom row) for $2000$ randomly generated features.}
    \label{fig:histograms}
\end{figure}

%%%%%%%%%%%%%%%%%     Theory

\section{Theoretical Guarantees}\label{theory}
We now provide the generalization guarantees of \alg. The following assumptions are used for the derivation of our result.  
\begin{assumption}\label{A1}
The loss function $y\mapsto L(y,\cdot)$ is uniformly G-Lipschitz-continuous in the first argument. 
\end{assumption}
A number of commonly used loss functions satisfy the assumption above. Examples include the logistic loss $L(y,y')=\log(1+\exp(-yy'))$ and hinge loss $L(y,y')=\max\{0,1-yy'\}$ for classification, and the quadratic loss $L(y,y')=(y-y')^2$ for regression.
\begin{assumption}\label{A2}
The feature map $\phi(\cdot,\cdot)$ satisfies $\sup_{\xb,\omegab}\abs{\phi(\xb,\omegab)} \leq 1$. This also implies $\sup_{\xb,\xb'}\abs{k(\xb,\xb')} \leq 1$ due to \eqref{integralform}.
\end{assumption}
Boundedness assumption is also standard (see e.g. \cite{rahimi2009weighted}). For example, cosine or sigmoidal feature maps (activation functions) satisfy the assumption. In general, when $\Xc$ and $\Omega$ are compact, the feature map can be normalized to satisfy Assumption \ref{A2}.   
 
 Algorithm \ref{ALGO1} aims to approximate a distribution on an infinite-dimensional set ($\Omega$) depending on an infinite-dimensional dimensional operator ($\Sigma$). The main two challenges in analyzing \alg~is that we construct such distribution with finite data and finite random features. After the following definition, we state our main result, in which we use $\tilde{O}(\cdot)$ to hide poly-log factors.
 
\begin{definition}\label{D3}(Degrees of freedom \cite{bach2017equivalence})
For a positive-definite operator $\Sigma$, degrees of freedom is defined as $\text{deg}_\lambda(\Sigma) \triangleq \tr{\Sigma(\Sigma+\lambda I)^{-1}}$.
\end{definition}
\begin{theorem}\label{thm}
Let Assumptions \ref{A1}-\ref{A2} hold. For a fixed parameter $\Delta \in (0,0.5]$, let $N_0 \geq \frac{8}{3}\Delta^{-2}\tilde{O}(M_0)$, $M_0=o(N)$, and $\lambda=\frac{1}{N}$ in Algorithm \ref{ALGO1}.
Let $M$ random features sampled from $\qbh$ (the output of Algorithm \ref{ALGO1}) define the class $\Fch_{\qbh}$ in \eqref{WRFclass}, and let $\fh_{\thetabh}$ be the minimizer of the empirical risk over $\Fch_{\qbh}$. If $M\geq 5\text{deg}_\lambda(\Sigmah)\log[16N\text{deg}_\lambda(\Sigmah)]$, for $g_{\gammab}\in \Fc_k$, we have
\begin{align*}
    \ex{R(\fh_{\thetabh})}-\underset{\norm{\gammab}^2\leq \frac{F}{N}}{\inf} R(g_{\gammab}) \leq {\tt Er(N)} + {\tt Er(k)}+ {\tt Er(g)},
    \end{align*}
where the expectation is over data and random features, $\Sigmah$ is the integral operator defined with respect to $\kh_{M_0}(\cdot,\cdot)$ in \eqref{MC}, $F$ is a constant factor, and
\begin{align*}
    &{\tt Er(N)}=O\left(\frac{1}{\sqrt{N}}\right)\\
    &{\tt Er(k)}=O\left(\sqrt{\E_{dp,dp'}\left[\text{var}_{d\tau}\left(\kh_{M_0}(\xb,\xb')\right)\right]}\right)\\
    &{\tt Er(g)}=O\left(\sqrt{\frac{\Delta}{N} \E_{d\tau}\left[\frac{1}{\sigma_{M_0}(\Gb)}\right]}\right).
\end{align*}
\end{theorem}
{\bf Interpretation:} The bound in Theorem \ref{thm} consists of three terms. ${\tt Er(N)}$ appears from calculating Rademacher complexity (estimation due to finite sample size $N$) and the choice of $\lambda=1/N$ (which turns out to be optimal in view of \eqref{e3}). ${\tt Er(k)}$ depends on the variance of the approximation of kernel $k$. Not only does it scale inversely with $M_0$, but also it depends on the feature map. On the other hand, ${\tt Er(g)}$ captures the impact of the minimum eigenvalue of $\Gb$ \eqref{Gb}. This quantity is a random number based on the choice of random features (but it is expected out over $d\tau$ in the bound). In general, the trade-off between ${\tt Er(g)}$ and ${\tt Er(k)}$ is structure-dependent and non-trivial. Increasing $M_0$ can improve ${\tt Er(k)}$ at the cost of making ${\tt Er(g)}$ looser. 

As defined in Section \ref{Problem4} both $k(\cdot,\cdot)$ and $g(\cdot,\cdot)$ depend on the feature map $\phi(\cdot,\cdot)$, but the important insight is that

``${\tt Er(g)}$ (which depends on $g$) is characterized by the data distribution $dp$, whereas ${\tt Er(k)}$ ((which depends on $g$) is characterized by the distribution of random features $d\tau$.''
\begin{example}
For Gaussian kernel using Monte Carlo sampling (see Lemma 2 in \cite{felix2016orthogonal}), the variance is
$$
\frac{1}{2M_0}\left(1-e^{-z^2}\right)^2 \approx \frac{1}{2M_0}z^4,
$$
for small $z$, where $z=\norm{\xb-\xb'}\sqrt{\text{var}_{d\tau}(\omegab)}$. If $z^4=O(M_0/N)$ (small variance for random features), then ${\tt Er(k)}=O(1/\sqrt{N})$. If $\sigma_{i}(\Gb)=\Theta(e^{-i})$, by choosing $M_0=\varepsilon\log N$ for $\epsilon\in(0,0.5)$, and letting $\Delta=N^{-\varepsilon}$, we can maintain the optimal bound as ${\tt Er(g)}=O(1/\sqrt{N})$. Similarly, if $\sigma_{i}(\Gb)=\Theta(i^{-1})$, $\Delta=M_0^{-1}$ can guarantee ${\tt Er(g)}=O(1/\sqrt{N})$.
\end{example}

The detailed proof of the theorem is in the supplementary material. It combines several ideas with prior results in literature. To analyze the difference between $q(\omegab)$ and $\qh(\omegab)$, we use the spectral approximation result of \cite{avron2017random} for ridge leverage functions (based on recent works on matrix concentration inequalities \cite{tropp2015introduction}). We further employ the approximation error by \cite{bach2017equivalence} for bounding the error caused by selecting $M$ random features out of $M_0$ possible samples $\{\tilde{\omegab}_m\}_{m=1}^{M_0}$.

\begin{remark}
Notice that the bound in Theorem \ref{thm} is for a Lipschitz continuous loss (similar to that of \cite{bach2017equivalence}), whereas the results in
\cite{rudi2016generalization} and \cite{sun2018but} are focused on ridge regression and SVM, respectively. On the other hand, our bound is in expectation, whereas the results in \cite{rudi2016generalization,sun2018but} are in high probability. The major difference (our contribution) is that all three prior works assumed (i) availability of $q(\cdot)$ \eqref{optimaldist} and (ii) the possibility of sampling from it. Our work relaxes these two by constructing and sampling from $\qh(\cdot)$ (Algorithm \ref{ALGO1}).
\end{remark}

\begin{remark}
Given that Theorem \ref{thm} relates the generalization bound via ${\tt Er(k)}$ to the variance of the kernel approximation, methods for variance reduction in sampling initial $M_0$ random features may be more effective than Monte Carlo. For example, Orthogonal Random Features (ORF) \cite{felix2016orthogonal} is a potential technique for variance reduction in approximation of the Gaussian kernel. In general, assuming a structure on the kernel (more than the integral form \eqref{integralform}) can result in more explicit error term ${\tt Er(k)}$ in the generalization bound, but pursuing this direction is outside of the scope of this work.
\end{remark}

%%%%%%%%%%%%%%%  Related Literature 

\section{Related Literature}\label{RL}
{\bf Random features:} The idea of randomized features  was proposed as an elegant technique for improving computational efficiency of kernel methods \cite{rahimi2007random}. As previously mentioned, a wide variety of kernels (of the form \eqref{integralform}), can be approximated using random features (e.g. shift-invariant kernels using Monte Carlo \cite{rahimi2007random} or Quasi Monte Carlo \cite{yang2014quasi} sampling, and dot product kernels \cite{kar2012random}. To further increase the efficiency with respect to the input dimension, a number of methods have been developed based on the properties of dense Gaussian random matrices (see e.g. Fast-food \cite{le2013fastfood} and Structured Orthogonal Random Features \cite{felix2016orthogonal}). These methods can decrease the time complexity by a factor of $O((\log d)/d)$. To study supervised learning, \cite{yen2014sparse} showed that using $\ell_1$-regularization combined with randomized coordinate descent, random features can be made more efficient. More specifically, to achieve $\epsilon$ error on the risk, $O(1/\epsilon)$ random features is required in contrast to $O(1/\epsilon^2)$ in the early work of \cite{rahimi2009weighted}. In the similar spirit and more recently, \cite{rudi2016generalization} showed that to achieve $O(1/\sqrt{N})$ learning error in ridge regression, only $M=O(\sqrt{N}\log N)$ random features is required.

 {\bf Data-dependent random features:} A number of recent works have focused on {\it kernel approximation} techniques based on data-dependent sampling of random features. Examples include \cite{yu2015compact} on compact nonlinear feature maps, \cite{yang2015carte,oliva2016bayesian} on approximation of shift-invariant/translation-invariant kernels, \cite{chang2017data} on Stein effect in kernel approximation, and \cite{agrawal2018data} on data-dependent approximation using greedy approaches (e.g. Frank-Wolfe). 
 
Another line of research has focused on {\it generalization} properties of data-dependent sampling. We discussed the {\it unsupervised} techniques based on leverage scores in the Introduction (e.g. \cite{bach2017equivalence,rudi2016generalization,sun2018but}). On the other hand, there are {\it supervised} methods \cite{sinha2016learning,AAAI1816684,bullins2017not} with the goal of improvement of test error. \cite{sinha2016learning} develop an optimization-based method to re-weight random features and sample important ones for better generalization. This method outperforms \cite{rahimi2009weighted} in the experiments, but the theoretical bound still indicates the need for $O(N)$ features to achieve $O(1/\sqrt{N})$ learning error. In a similar fashion, \cite{AAAI1816684} propose a (supervised) score function for resampling of random features. While effective in practice compared to its prior works, the method does not have a theoretical generalization guarantee. \cite{bullins2017not} study data-dependent approximation of  translation-invariant/rotation-invariant kernels with a focus on SVM. Their technique works based on maximizing the kernel alignment in the Fourier domain. The theoretical bound is derived by applying no-regret learning to solve SVM dual. We finally remark that recently \cite{li2018unified} have provided analysis of \alg~ in the case of Ridge regression. However, our results are valid for Lipschitz losses and the generalization bound is different. In particular, our bound depends on the eigenvalue decay of the (random) feature gram matrix, highlighting the role of data distribution.

{\bf Taylor (explicit) features:} Beside random features, explicit feature maps have also been used in speeding up kernel methods. Cotter et al \cite{cotter2011explicit} discuss the Taylor approximation of Gaussian kernel in training SVM and provide empirical comparisons with random features. Low-dimensional Taylor approximation has also been addressed in \cite{yang2004efficient,xu2006explicit} for Gaussian kernel as well as in \cite{minh2006mercer} for other practical kernels. Furthermore, the authors of \cite{vedaldi2012efficient} quantify the approximation error of additive homogeneous kernels. Finally, greedy approximation using explicit features has been discussed in \cite{shahrampour2018learning}. In general, the experiments of \cite{cotter2011explicit} for Gaussian kernel suggests that in comparison of Taylor vs random features, none clearly dominates the other, as the structure of data indeed plays an important role in having a better fit.

{\bf Nystr{\"o}m method:} This work is also relevant to Nystr{\"o}m method which offers a data-dependent sampling scheme for kernel approximation \cite{williams2001using,drineas2005nystrom}. In this approach, we use a subset of training data to approximate a surrogate kernel matrix, and then we transform the data points using the approximated kernel.  Though being data-dependent, the main difference of this line of research with this work is that we are concerned with learning good features for generalization.

%%%%%%%%%%      Simulations

\begin{figure*}[t]
    \centering
    \includegraphics[width=1\linewidth]{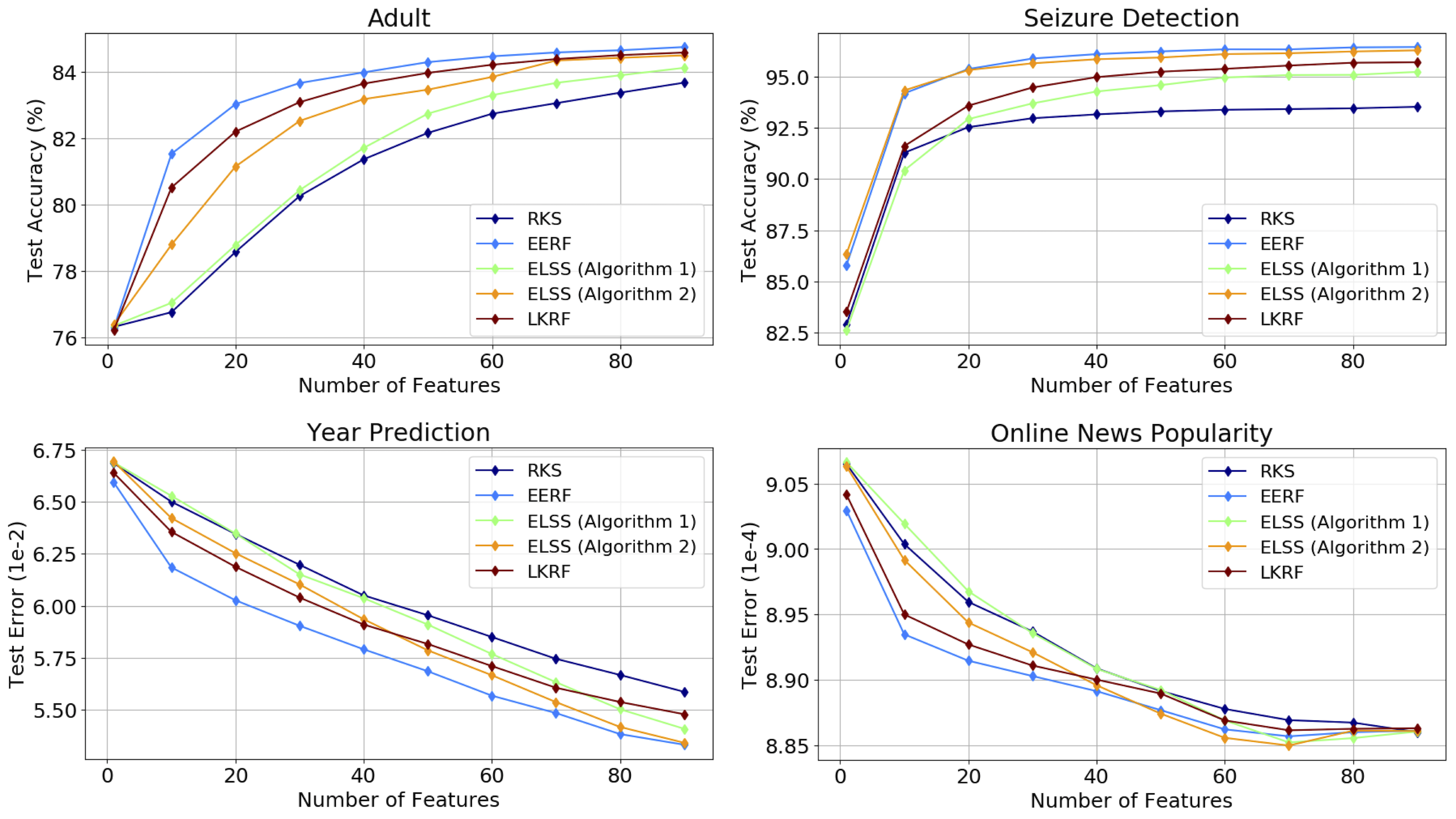}
    \caption{Comparison of the test error of Algorithm 1 and Algorithm 2 against randomized features baselines RKS, LKRF, and EERF.}
    \label{fig:results}
\end{figure*}

\section{Empirical Evaluations}\label{simulations}
In this section, we provide numerical experiments on four datasets from the UCI Machine Learning Repository.

{\bf Benchmark algorithms:} We use the following methods in randomized kernel approximation as baselines:\\ 
{\bf 1) RKS \cite{rahimi2009weighted}}, with approximated Gaussian kernel: $\phi=\cos(\xb^\top\omegab_m+b_m)$ in \eqref{MC}, $\{\omegab_m\}_{m=1}^M$ are sampled from a Gaussian distribution, and $\{b_m\}_{m=1}^M$ are sampled from the uniform distribution on $[0,2\pi)$.\\
{\bf 2) LKRF \cite{sinha2016learning}}, with approximated Gaussian kernel: $\phi=\cos(\xb^\top\omegab_m+b_m)$. $M_0$ random features $(M_0>M)$ are sampled and re-weighted by solving a kernel alignment optimization. The top $M$ features are used in the training.\\
{\bf 3) EERF \cite{AAAI1816684}}, with approximated Gaussian kernel: $\phi=\cos(\xb^\top\omegab_m+b_m)$, and similar to LKRF, $M_0>M$ number of initial random features are sampled and then re-weighted according to a score function. The top $M$ random features are used in the training. 

The selection of the baselines above allows us to evaluate the generalization performance of one \emph{data-independent} method (\cite{rahimi2009weighted}) and two \emph{supervised} data-dependent methods (\cite{sinha2016learning,AAAI1816684}) for sampling random features, and compare them to \alg, which is an \emph{unsupervised} data-dependent method. It should be noted that LKRF and EERF learn a new kernel based on input-outputs, but in view of \eqref{kernelq}, \alg~only performs importance sampling and does not change the kernel. To change the kernel (still in an {\it unsupervised} manner) we modify \alg~(Algorithm \ref{ALGO1}) to choose the top $M$ features (out of $M_0$) that have the most weight without actually sampling them. We present that as \alg~(Algorithm 2) in the experiments.

\begin{figure}[t]
    \centering
    \includegraphics[width=0.7\columnwidth]{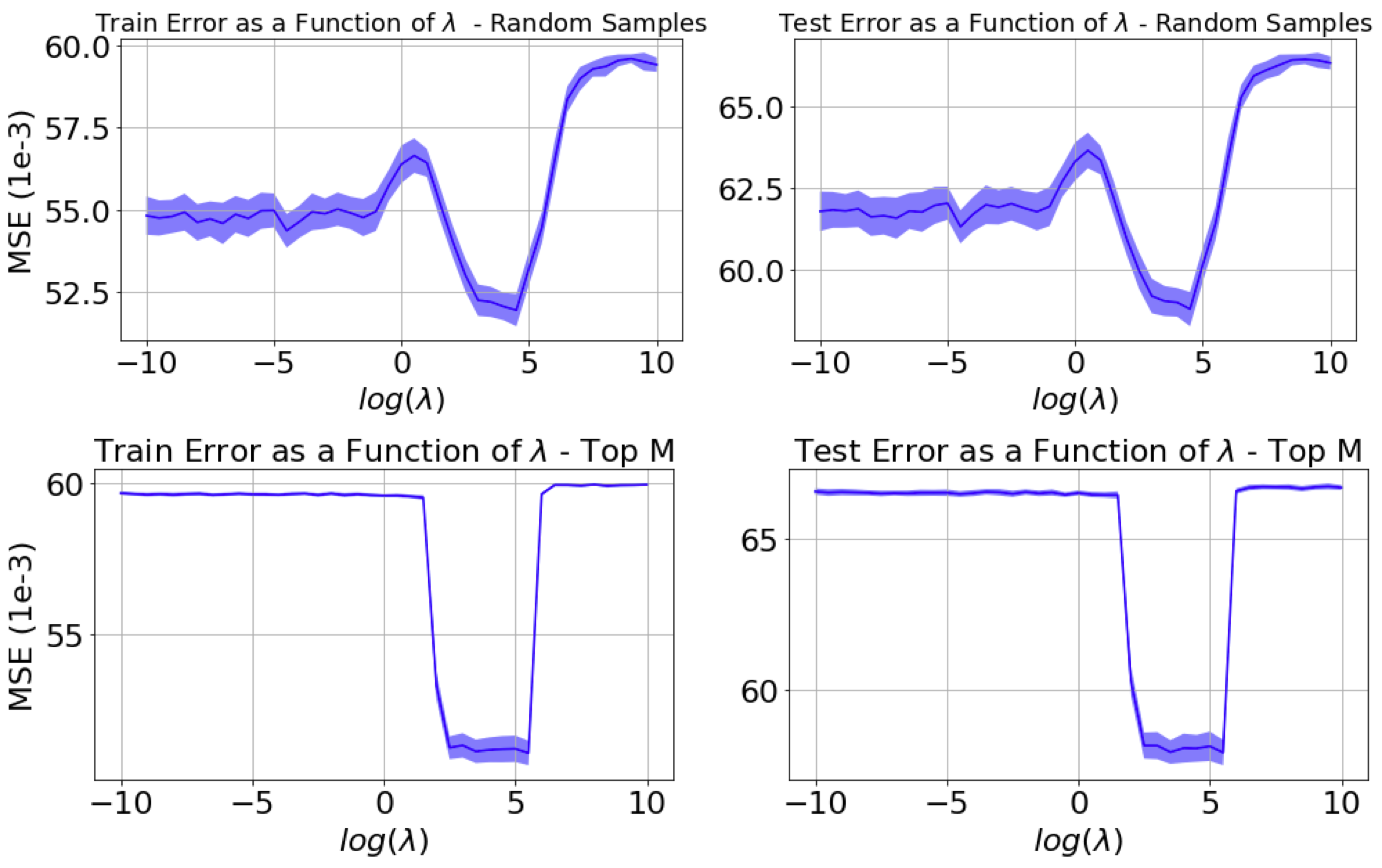}
    \caption{The change in train and test error rates with respect to $\lambda$ for the Year Prediction dataset. Top and bottom rows show ELSS Algorithm 1 and Algorithm 2, respectively.}
    \label{fig:lambda}
\end{figure}

{\bf Practical considerations:} The Python code of our paper is available on Github \footnote{https://github.com/.../ELSS (Suppressed for double-blind review)}. Grid search was performed to obtain the optimal hyper-parameter of each method for each dataset. For instance, to determine the width of the Gaussian kernel $K(\xb,\xb')=\exp(-\norm{\xb-\xb'}^2/2\sigma^2)$, we obtain the value of $\sigma$ for each dataset using grid search in $[1e-10,1e+3]$. Notice that for randomized approaches, this amounts to sampling random features from $\sigma^{-1}\Nc(0,I_d)$. Following the work of \cite{AAAI1816684} and as a rule of thumb, we set $M_0=10M$ for all algorithms. Theorem \ref{thm} suggests that $N_0>\Delta^{-2}M_0$, and given that in our experiments $M$ can go up to $100$, even for $\Delta\approx0.2$, $N_0\approx 25000$, so we simply use $N_0=N$ for each dataset. The hyper-parameters of the optimization step in LKRF \cite{sinha2016learning} are tuned and the best results are reported. 

{\bf Datasets:}
In this work we used four datasets from the UCI Machine Learning Repository, namely the Year Prediction, Online News Popularity, Adult, and Epileptic Seizure Detection datasets, where the former two datasets are regression tasks and the latter two are binary classification tasks.  Table \ref{table1} tabulates the information for each dataset. If the training and test samples are not provided separately for a dataset, we split it randomly. We standardize the data by scaling the features to have zero mean and unit variance and the responses in regression to be inside $[-1,1]$.
\begin{table}[h!]
\caption{Input dimension, number of training samples, and number of test samples are denoted by $d$, $N_\text{train}$, and $N_\text{test}$, respectively.}
\begin{center}
\resizebox{0.6\columnwidth}{!}{
\begin{tabular}{| c ||  c | c | c | c | @{}m{0pt}@{}} 
 \hline 
  Dataset &  Task   & $d$ & $N_\text{train}$ & $N_\text{test}$ &\\ [1.5 ex]
 \hline  
 Year prediction &  Regression  &  90  & 46371 & 5163 &\\ [1.5 ex]
 \hline
Online news popularity &  Regression &  58   & 26561 & 13083 &\\ [1.5 ex]
 \hline
Adult & Classification & 122     & 32561 & 16281  &\\ [1.5 ex]
 \hline
Epileptic seizure recognition  & Classification  & 178  & 8625 & 2875  &\\ [1.5 ex]
\hline
\end{tabular}\label{table1}}
\end{center}
\end{table}

{\bf Performance:} The results on datasets in Table \ref{table1} are reported in Figure \ref{fig:results}. Each experiment was repeated $50$ times and the average generalization performance (i.e., test accuracy/error) of the methods are reported. It can be seen that ELSS Algorithm 1 performs better than RKS, which is data-independent. ELSS Algorithm 2 boosts the performance even further and brings the performance closer to that of supervised data-dependent methods, i.e., EERF and LKRF, especially for $M=80-100$. It is interesting to observe that in Year Prediction and Seizure Detection, ELSS Algorithm 2, which changes the kernel unsupervised, outperforms LKRF.

{\bf Sensitivity to $\lambda$:} 
Naturally a question arises regarding the sensitivity of the generalization performance of \alg~with respect to $\lambda$. From a theoretical point of view and for shift-invariant kernels, one expects to see uniformly distributed weights (i.e., equivalent to RFF) for too large and too small values of $\lambda$ and a sweet spot in between. To confirm this, we calculated the train and test error of \alg~ for the Year Prediction dataset and reported the results in \ref{fig:lambda}. The top and bottom rows correspond to \alg~ Algorithm 1 and Algorithm 2. Each experiment for each $\lambda$ was repeated 50 times and the mean and standard deviations are reported. 

\begin{figure}
    \centering
    \includegraphics[width=\linewidth]{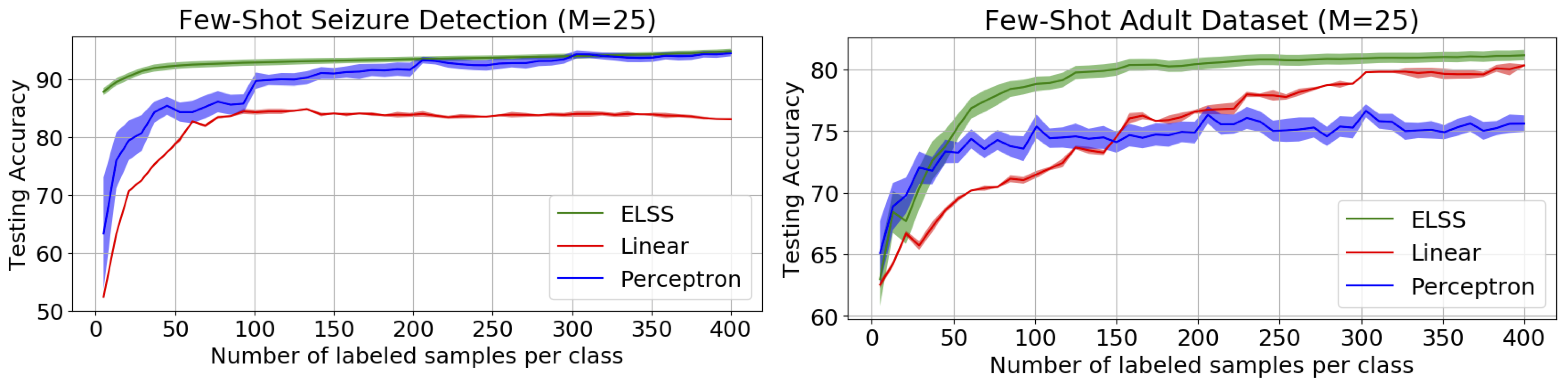}
    \caption{The results of few-shot learning with ELSS with $M=25$ random features, logistic regression (Linear), and a perceptron with $M=25$ latent nodes for the Seizure Detection and Adult datasets.}
    \label{fig:lwll}
\end{figure}

{\bf Learning with Less Labels (LwLL):} The existing state-of-the-art machine learning models, and specifically, deep learning architectures are data hungry and require a large number of labeled samples for training. Learning with few labels has been a long standing goal in the ML community. Semi-supervised learning, active-learning, and more recently zero-shot, one-shot, and few-shot learning paradigms study different aspects of this problem. Here, we show that the unsupervised nature of ELSS enables us to perform efficient LwLL.

 We consider the scenario in which we have lots of unlabeled data with few labeled samples as our training set. To that end, for the Seizure Detection and Adult datasets we use only $K\in\{5,10,...,400\}$ labeled samples per class for training. We then perform classification using Logistic Regression (LR), ELSS+LR, and a neural network (i.e. a perceptron). For ELSS we used $M=25$ random features and for the perceptron we used $M=25$ latent neurons. We repeated the experiments for each classifier $100$ times (with randomized sets of training samples) and measured the testing accuracy. The mean and standard deviation of the testing accuracy of these models for different number of $K$'s is depicted in Figure \ref{fig:lwll}. Note that comparison of ELSS+LR and LR serves as an ablation study and shows the benefit of our proposed approach. In addition, comparison of ELSS+LR with the perceptron shows the benefit of our proposed method compared to neural networks in the LwLL setting.

{\bf Concluding remarks:}  A main distinction between leverage scores and the existing literature on data-dependent random feature generation, is the \emph{unsupervised} nature of \alg. More interestingly, \alg~can provide generalization performance that is on par with \emph{supervised} methods, which use input-output pairs for random feature generation, e.g., \cite{sinha2016learning} and \cite{AAAI1816684}. But why is it important to have an unsupervised data-dependent feature generator, specifically, when the final task is supervised learning? The answer lies in the realm of learning with less labels (LwLL). In supervised LwLL, one cannot afford to train complex classifiers due to the lack of enough labeled data. While linear classifiers generally perform poorly on ``complex'' datasets. In such scenarios, \alg~ could leverage large number of unlabeled data and extract features that provide similar generalization performances to the ones extracted with full supervision. A linear classifier then can be trained on the extracted features with few labels.

%\section*{Acknowledgements}
%We gratefully acknowledge the support of ...

\section{Appendeix}

\subsection{Estimation error}
We start with the definition of Rademacher complexity, which is quite standard, but we provide it for completeness. 
\begin{definition}(Rademacher complexity \cite{bartlett2002rademacher})\label{D1}
For a finite-sample set $\{\xb_n\}_{n=1}^N$, the empirical Rademacher complexity of a class $\Fc$ is defined as 
$$
\widehat{\Rc}(\Fc) \triangleq \frac{1}{N} \E_{\sigmab}\left[\sup_{f\in \Fc}\sum\limits_{n=1}^N\sigma_nf(\xb_n)\right],
$$
where the expectation is taken over $\{\sigma_i\}_{i=1}^N$ that are independent samples uniformly distributed on the set $\{-1,1\}$. The Rademacher complexity is then $\Rc(\Fc)\triangleq\E_{dp}\widehat{\Rc}(\Fc)$.
\end{definition}

\subsection{The error of approximating $\tilde{\Gb}$ with finite data $N_0$}

Next, we have the notion of spectral approximation, which will be used in the proof of our main result. 
\begin{definition}\label{D2}($\Delta$-spectral approximation \cite{avron2017random})
A matrix $\Ab$ is a $\Delta$-spectral approximation of another matrix $\Bb$, if the following relationship holds $$(1-\Delta)\Bb \preccurlyeq \Ab \preccurlyeq (1+\Delta)\Bb.$$
\end{definition}
We now provide the following theorem by \cite{avron2017random} on spectral approximation. Note that to avoid confusion, we re-write the theorem with the notation in this work. In particular, observe that in \cite{avron2017random} the kernel matrix is defined for data with respect to random features (similar to $\Kb$ in \eqref{Kb}), whereas we state the result for $\Gb$ which is defined for random  features with respect to data\footnote{The role of random features and inputs are interchanged.}.
For the sake of simplicity in presentation we use $\Phib$ instead of $\Phib_{M_0,N_0}$.
\begin{theorem}\cite{avron2017random}\label{Musco}
Let $\Delta \in (0,0.5]$ and $\delta\in(0,1)$. Assume that $\norm{\Gb}\geq \mu$. If we use $N_0 \geq \frac{8}{3}\Delta^{-2} \frac{M_0}{\mu}\log\delta^{-1}$ random samples from $dp$, then $\Phib\Phib^\top+\mu\Ib$ is a $\Delta$-spectral approximation of $\Gb+\mu\Ib$ with probability of at least $1-\delta$.
\end{theorem}
We dropped an $o(M_0)$ term inside the logarithm argument above (which does not affect our result). We now use Theorem \ref{Musco} to obtain the spectral approximation of the kernel matrix $\Gb$ in \eqref{Gb}. Observe that $\Gb$ is normalized with its dimension (unlike \cite{avron2017random}), which necessitates refinement of some parameters in the theorem above before applying it.
\begin{lemma}\label{L1}
Let $\Delta \in (0,0.5]$ and $\delta\in(0,1)$. Given a fixed scalar $C^2$, let $\lambda=\frac{C^2}{N}$. Then, with probability at least $1-\delta$, we have that $(\frac{\Phib\Phib^\top}{\lambda}+\Ib)^{-1}$ is a $2\Delta$-spectral approximation of  $(\frac{\Gb}{\lambda}+\Ib)^{-1}$, when $N_0 \geq \frac{8}{3}\Delta^{-2}M_0\log\delta^{-1}$ and $M_0=o(N)$.
\end{lemma}
\begin{proof}
Observe that 
\begin{align*}
\ex{\frac{\Phib\Phib^\top}{\lambda}}=\frac{\Gb}{\lambda},
\end{align*}
so we should apply Theorem \ref{Musco} for $\mu=1$. 
First, we should check the condition $\norm{\frac{\Gb}{\lambda}}\geq 1$. Since
$$
M_0\Gb=\begin{bmatrix}
  g(\omegab_1,\omegab_1) & \cdots &   g(\omegab_1,\omegab_{M_0})\\
  \vdots &   \vdots &   \vdots  \\ 
    g(\omegab_{M_0},\omegab_1) & \cdots & g(\omegab_{M_0},\omegab_{M_0})
   \end{bmatrix},
$$
and the right-hand side is the gram matrix, which has a positive norm $\kappa$ independent of $N$. Then, we should verify 
$$
\kappa \geq \lambda M_0=\frac{C^2M_0}{N},
$$
which holds since $M_0=o(N)$. Now with $\mu=1$, we need $N_0 \geq \frac{8}{3}\Delta^{-2}M_0\log\delta^{-1}$ samples to have 
$$
(1-\Delta)\left(\frac{\Gb}{\lambda}+\Ib\right)   
\preccurlyeq
\left(\frac{\Phib\Phib^\top}{\lambda}+\Ib\right)
\preccurlyeq
(1+\Delta)\left(\frac{\Gb}{\lambda}+\Ib\right),  
$$
which by simple algebra implies
\begin{align*}
    (1-2\Delta)\left(\frac{\Gb}{\lambda}+\Ib\right)^{-1}   \preccurlyeq
    \left(\frac{\Phib\Phib^\top}{\lambda}+\Ib\right)^{-1}
      \preccurlyeq
    (1+2\Delta)\left(\frac{\Gb}{\lambda}+\Ib\right)^{-1},
\end{align*}
when $\Delta \in (0,0.5]$.
\end{proof}
\begin{lemma}\label{L4}
Let $\Delta \in (0,0.5]$ and $\delta\in(0,1)$. Given a fixed scalar $C^2$, let $\lambda=\frac{C^2}{N}$. Recall the definition of $\tilde{\Gb}$ and $\Qb$ in \eqref{trueridge} and \eqref{Q}, respectively. Then, with probability at least $1-\delta$ (over $N_0$ data points), we have that
$$
\sum\limits_{i=1}^{M_0}\abs{[\Qb]_{ii}-[\tilde{\Gb}]_{ii}}\leq 2\Delta \sum\limits_{i=1}^{M_0}\frac{\lambda}{\lambda+\sigma_i(\Gb)},
$$
as long as $N_0 \geq \frac{8}{3}\Delta^{-2}M_0\log\delta^{-1}$ and $M_0=o(N)$.
\end{lemma}
\begin{proof}
Let us start with the fact that
\begin{align*}
\tilde{\Gb}=\Gb(\Gb+\lambda\Ib)^{-1}=\Ib-\left(\frac{\Gb}{\lambda}+\Ib\right)^{-1},
\end{align*}
and 
\begin{align*}
\Qb&=\Phib\Phib^\top\left(\Phib\Phib^\top+\lambda\Ib\right)^{-1}=\Ib-\left(\frac{\Phib\Phib^\top}{\lambda}+\Ib\right)^{-1}.
\end{align*}
Therefore, since 
$$
 \tilde{\Gb}-\Qb=\left(\frac{\Phib\Phib^\top}{\lambda}+\Ib\right)^{-1}
    -\left(\frac{\Gb}{\lambda}+\Ib\right)^{-1},
$$
due to Lemma \ref{L1}, we derive
\begin{align}\label{E10}
    -2\Delta\left(\frac{\Gb}{\lambda}+\Ib\right)^{-1}   \preccurlyeq
    \tilde{\Gb}-\Qb
      \preccurlyeq
    2\Delta\left(\frac{\Gb}{\lambda}+\Ib\right)^{-1},
\end{align}
entailing
\begin{align*}
\sum\limits_{i=1}^{M_0}\abs{[\Qb]_{ii}-[\tilde{\Gb}]_{ii}}&\leq 2\Delta\tr{\left(\frac{\Gb}{\lambda}+\Ib\right)^{-1}}\\
&=2\Delta \sum\limits_{i=1}^{M_0}\frac{\lambda}{\lambda+\sigma_i(\Gb)},
\end{align*}
which finishes the proof.
\end{proof}
The above lemma is used to bound the total variation distance between $\qbh$ and $\qb$, as probability mass functions over $\{\tilde{\omegab}_m\}_{m=1}^{M_0}$. Note that from Section 4.2 of \cite{bach2017equivalence}, the leverage score $s(\omegab_i)=[\tilde{\Gb}]_{ii}$ for $i\in [M_0]$. As a result, the optimized distribution $q(\omegab)$ with respect to the uniform measure on $\{\tilde{\omegab}_m\}_{m=1}^{M_0}$ is
\begin{align}\label{E18}
q(\tilde{\omegab}_i)=\frac{[\tilde{\Gb}]_{ii}}{\tr{\tilde{\Gb}}}, ~~~\text{for all $i\in [M_0]$},
\end{align}
and from Algorithm \ref{ALGO1} recall that
\begin{align*}
\qh(\tilde{\omegab}_i)=\frac{[\Qb]_{ii}}{\tr{\Qb}}, ~~~\text{for all $i\in [M_0]$},
\end{align*}
\begin{corollary}\label{C1}
Let $\Delta \in (0,0.5]$ and $\delta\in(0,1)$. Given a fixed scalar $C^2$, let $\lambda=\frac{C^2}{N}$. Recall the definition of $\Gb$ in \eqref{Gb}. Then, with probability at least $1-\delta$ (over $N_0$ data points), we have that
$$
\sum\limits_{m=1}^{M_0}\abs{q(\tilde{\omegab}_m)-\qh(\tilde{\omegab}_m)}\leq 4\Delta \frac{\sum\limits_{m=1}^{M_0}\frac{\lambda}{\lambda+\sigma_m(\Gb)}}{\sum\limits_{m=1}^{M_0}\frac{\sigma_m(\Gb)}{\lambda+\sigma_m(\Gb)}},
$$
as long as $N_0 \geq \frac{8}{3}\Delta^{-2}M_0\log\delta^{-1}$ and $M_0=o(N)$.
\end{corollary}
\begin{proof}
Let us start with
\begin{align*}
    \sum\limits_{m=1}^{M_0}\abs{q(\tilde{\omegab}_m)-\qh(\tilde{\omegab}_m)}&=\sum\limits_{i=1}^{M_0}\abs{\frac{[\Qb]_{ii}}{\tr{\Qb}}-\frac{[\tilde{\Gb}]_{ii}}{\tr{\tilde{\Gb}}}}\\
    &\leq \sum\limits_{i=1}^{M_0}\abs{\frac{[\tilde{\Gb}]_{ii}}{\tr{\tilde{\Gb}}}-\frac{[\Qb]_{ii}}{\tr{\tilde{\Gb}}}}\\
    &+\sum\limits_{i=1}^{M_0}\abs{\frac{[\Qb]_{ii}}{\tr{\Qb}}-\frac{[\Qb]_{ii}}{\tr{\tilde{\Gb}}}}.\\
\end{align*}
Since $[\Qb]_{ii}\geq 0$, the second term in the bound above simplifies to 
$$
\sum\limits_{i=1}^{M_0}[\Qb]_{ii}\abs{\frac{1}{\tr{\Qb}}-\frac{1}{\tr{\tilde{\Gb}}}}=\abs{\frac{\tr{\Qb}-\tr{\tilde{\Gb}}}{\tr{\tilde{\Gb}}}},
$$
which is smaller than the first term. Thus, we get
\begin{align*}
    \sum\limits_{m=1}^{M_0}\abs{q(\tilde{\omegab}_m)-\qh(\tilde{\omegab}_m)}\leq 2\sum\limits_{i=1}^{M_0}\abs{\frac{[\tilde{\Gb}]_{ii}}{\tr{\tilde{\Gb}}}-\frac{[\Qb]_{ii}}{\tr{\tilde{\Gb}}}}\\
    \leq 2\frac{\Delta}{\tr{\tilde{\Gb}}} \sum\limits_{i=1}^{M_0}\frac{\lambda}{\lambda+\sigma_i(\Gb)},
\end{align*}
where we applied Lemma \ref{L4}. Observing that 
$$
\tr{\tilde{\Gb}}=\sum\limits_{m=1}^{M_0}\frac{\sigma_m(\Gb)}{\lambda+\sigma_m(\Gb)},
$$
and plugging it in the bound completes the proof. 
\end{proof}

\subsection{The error of approximation using $M$ random features out of $M_0$}

For bounding the error caused by selecting $M$ random features out of $M_0$ possible samples $\{\tilde{\omegab}_m\}_{m=1}^{M_0}$, we use the approximation error by \cite{bach2017equivalence}, which is adopted in our notation, following the subsequent definitions. For any probability mass function $\qb$, we can define the following class of functions:
\begin{align}\label{E4}
\Fch_{\qb} \triangleq \left\{ f(\cdot)=\sum\limits_{m=1}^M  \theta_{m}\frac{\phi(\cdot,\omegab_m)}{\sqrt{q(\omegab_m)}}: \thetab \in \R^M\right\}.
\end{align}
Also, let $\kh_{M_0}(\xb,\xb')\triangleq\inn{\phib_{M_0}(\xb),\phib_{M_0}(\xb')}$ be the kernel approximated using $M_0$ random features sampled from $d\tau$. Then,
\begin{align}\label{E41}
\Fc_{\kh_{M_0}}\triangleq\left\{f(\cdot)=\sum\limits_{n=1}^N \alpha_n \kh_{M_0}(\xb_n,\cdot): \alphab \in \R^N\right\}.
\end{align}
The following result \cite{bach2017equivalence} characterizes the (minimum) distance between these two classes.
\begin{proposition}\label{Bach} (Approximation of the unit ball of $\Fc_{\kh_{M_0}}$ for optimized distribution \cite{bach2017equivalence}) For $\lambda>0$ and the distribution with density $q(\omegab)$ defined in equation \eqref{E18} with respect to $d\tauh$ (the uniform measure on $\{\tilde{\omegab}_m\}_{m=1}^{M_0}$). Let $\{\omegab_m\}_{m=1}^M$ be sampled i.i.d. from the density $q(\omegab)$, defining the
kernel $\frac{1}{M}\sum_{m=1}^M q^{-1
}(\omegab_m)\phi(\xb ,\omegab_m)\phi(\xb',\omegab_m)$, and its associated RKHS $\Fch_{\qb}$ in \eqref{E4}.
Then, for any
$\delta\in(0,1)$, with probability $1-\delta$ with respect to samples $\{\omegab_m\}_{m=1}^M$, we have,
\begin{align}
    \underset{\norm{f}_{\Fc_{\kh_{M_0}}}\leq 1}{\sup}~~\underset{\norm{\fh}_{\Fch_{\qb}}\leq 2}{\inf}\norm{f-\fh}^2_{\Lc^2(dp,\Xc)} \leq 4\lambda,
\end{align}
if $M\geq 5\text{deg}_\lambda(\Sigmah)\log\frac{16\text{deg}_\lambda(\Sigmah)}{\delta}$, where $\text{deg}_\lambda(\Sigmah)$ is defined in Definition \ref{D3}, and $\Sigmah$ is the integral operator defined with respect to $\kh_{M_0}(\xb,\xb')$.
\end{proposition}
We remark that in \cite{bach2017equivalence}, the result above has been stated for comparing $\Fch_{\qb}$ with $\Fc_k$ under the assumption of denseness of $\Fc_k$ in $\Lc^2(dp,\Xc)$ to avoid zero eigenvalues of the operator. However, as mentioned in Section 2.1 of \cite{bach2017equivalence}, this assumption can be relaxed\footnote{One can generate a sequence of nonzero positive numbers that sum to an infinitesimal number.}. Since our base class is derived by the uniform measure on $\{\tilde{\omegab}_m\}_{m=1}^{M_0}$ (rather than whole $d\tau$), we can only compare $\Fch_{\qb}$ with $\Fc_{\kh_{M_0}}$ using \cite{bach2017equivalence}. In the next section, we compare $\Fc_{\kh_{M_0}}$ with $\Fc_k$.

\subsection{The approximation error of sampling $M_0$ random features from $d\tau$}

\begin{lemma}\label{L7} 
Recall from \eqref{kernelclass} that $$\Fc_k\triangleq\left\{f(\cdot)=\sum\limits_{n=1}^N \alpha_nk(\xb_n,\cdot): \alphab \in \R^N\right\},$$
and from \eqref{E41} that
$$\Fc_{\kh_{M_0}}\triangleq\left\{f(\cdot)=\sum\limits_{n=1}^N \alpha_n \kh_{M_0}(\xb_n,\cdot): \alphab \in \R^N\right\}.$$
Then, for any $g_{\alphab}\in \Fc_k$ and $\gh_{\alphab}\in \Fc_{\kh_{M_0}}$ such that
$\norm{\alphab}^2\leq \frac{F}{N}$, we have
{\small
$$
\E\norm{g_{\alphab}-\gh_{\alphab}}_{\Lc^2(dp,\Xc)}\leq \sqrt{F\E_{dp,dp'}\left[\text{var}_{d\tau}\left(\kh_{M_0}(\xb,\xb')\right)\right]},
$$}
where the expectation on the left-hand side is over data and random features, and the variance on the right-hand side is over random features. 
\end{lemma}
\begin{proof}
for any $g_{\alphab}\in \Fc_k$ and $\gh_{\alphab}\in \Fc_{\kh_{M_0}}$, we have
\begin{align*}
\norm{g_{\alphab}-\gh_{\alphab}}_{\Lc^2(dp,\Xc)}=
\norm{\sum\limits_{n=1}^N \alpha_nk(\xb_n,\xb)-\sum\limits_{n=1}^N \alpha_n \kh_{M_0}(\xb_n,\xb)}_{\Lc^2(dp,\Xc)}.
\end{align*}
Let $\eb(\xb)=[e_1(\xb),\ldots,e_N(\xb)]^\top$, where for $n\in [N]$
$$
e_n(\xb)\triangleq k(\xb_n,\xb)-\kh_{M_0}(\xb_n,\xb).
$$
Then since $\alphab^\top\eb \leq \norm{\alphab}\norm{\eb}$, for any $\norm{\alphab}^2\leq \frac{F}{N}$, we get
\begin{align*}
    \norm{g_{\alphab}-\gh_{\alphab}}_{\Lc^2(dp,\Xc)}&=\norm{\alphab^\top\eb(\xb)}_{\Lc^2(dp,\Xc)}\leq \sqrt{\E_{dp}\norm{\alphab}^2\norm{\eb(\xb)}^2}\leq \sqrt{\frac{F}{N}\E_{dp}\norm{\eb(\xb)}^2}.
\end{align*}
Observe that $\E_{d\tau}[e_n(\xb)]=0$ and so $\E_{d\tau}[e^2_n(\xb)]=\text{var}_{d\tau}(\kh_{M_0}(\xb_n,\xb))$. Taking expectation with respect to $d\tau$ from above and using Jensen's inequality, we get 
\begin{align*}
\E_{d\tau}\norm{g_{\alphab}-\gh_{\alphab}}_{\Lc^2(dp,\Xc)} \leq \sqrt{\frac{F}{N}\E_{dp}\sum\limits_{n=1}^N\text{var}_{d\tau}(\kh_{M_0}(\xb_n,\xb))}.
\end{align*}
Noting that $\{\xb_n\}_{n=1}^N$ are i.i.d. and taking another expectation from above with respect to the randomness of data, we have {\small
\begin{align*}
    \E\norm{g_{\alphab}-\gh_{\alphab}}_{\Lc^2(dp,\Xc)} &\leq \sqrt{\frac{F}{N}\E_{dp}\E_{dp'}\sum\limits_{n=1}^N\text{var}_{d\tau}(\kh_{M_0}(\xb',\xb))}\\
    &= \sqrt{F\E_{dp,dp'}\left[\text{var}_{d\tau}\left(\kh_{M_0}(\xb,\xb')\right)\right]},
\end{align*}}
which completes the proof. 
\end{proof}

\subsection{Proof of Theorem \ref{thm}}
Recall the the definition of $\Fch_{\qbh}$ in \eqref{WRFclass} and $\Fch_{\qb}$ in \eqref{E4}. Let us define 
$$
\fh_{\thetabh}\triangleq\underset{f\in \Fch_{\qbh}: \norm{\thetab}\leq \frac{2F}{\sqrt{M}}}{\argmin} \Rh(f).
$$
For any $\norm{\betab}^2\leq \frac{2F}{M}$, let $\fh_{\betab}$ be another function in $\Fch_{\qbh}$ and $f_{\betab}\in \Fch_{\qb}$. We now have for any $g_{\gammab}\in \Fc_{k}$ and $\gh_{\alphab}\in \Fc_{\kh_{M_0}}$ that
\begin{align*}
    \vphantom{+\underbrace{R(f_{\betab})-R(\gh_{\alphab})}_{e_3}}R(\fh_{\thetabh})&=R(\fh_{\thetabh})-\Rh(\fh_{\thetabh})+\Rh(\fh_{\thetabh})\\
    \vphantom{+\underbrace{R(f_{\betab})-R(\gh_{\alphab})}_{e_3}} &\leq R(\fh_{\thetabh})-\Rh(\fh_{\thetabh})+\Rh(\fh_{\betab})\\
&=\underbrace{R(\fh_{\thetabh})-\Rh(\fh_{\thetabh})+\sup_{\norm{\betab}^2\leq \frac{2F}{M}}\left[\Rh(\fh_{\betab})-R(\fh_{\betab})\right]}_{e_1}\\
 &+\underbrace{\sup_{\norm{\betab}^2\leq \frac{2F}{M}}\left[R(\fh_{\betab})-R(f_{\betab})\right]}_{e_2}\\
    &+\underbrace{\sup_{\alphab^\top\Kb\alphab\leq \frac{F}{N}}\inf_{\norm{\betab}^2\leq \frac{2F}{M}} \left[R(f_{\betab})-R(\gh_{\alphab})\right]}_{e_3}\\
    &+\underbrace{\sup_{\norm{\gammab}^2\leq \frac{F}{N}}\inf_{\norm{\alphab}^2\leq \frac{F}{N}} \left[R(\gh_{\alphab})-R(g_{\gammab})\right]}_{e_4}\\
    &+\inf_{\norm{\gammab}^2\leq \frac{F}{N}} R(g_{\gammab}). \numberthis \label{E3}
\end{align*}
The rest of the proof follows by bounding the terms above. 
\subsection*{Bounding $e_1$}
Standard arguments for Rademacher complexity of kernels (see e.g Lemma 22 in \cite{bartlett2002rademacher}) combined with Assumption \ref{A1} implies,
\begin{align*}
    e_1\leq \frac{4G\sqrt{F}}{\sqrt{M}N}\sqrt{\sum\limits_{n=1}^N\sum\limits_{m=1}^M \frac{\phi^2(\xb_n,\omegab_m)}{\qh(\omegab_m)}}.
\end{align*}
Since we sample $M_0$ random features $\{\tilde{\omegab}_m\}_{m=1}^{M_0}$ according to $d\tau$, it is useful to define the following measure 
\begin{align}\label{tauhat}
    d\tauh(\omegab)\triangleq\frac{1}{M_0}\sum\limits_{m=1}^{M_0}\delta(\omegab-\tilde{\omegab}_m),
\end{align}
where $\delta(\cdot)$ is the Dirac delta function. Taking expectation with respect to the measure $\qh(\omegab) d\tauh(\omegab)$ and using Jensen's inequality, we get
\begin{align*}
   \E_{\qh d\tauh}[e_1]\leq \frac{4G\sqrt{F}}{N}\sqrt{\frac{1}{M_0}\sum\limits_{n=1}^N\sum\limits_{m=1}^{M_0} \phi^2(\xb_n,\tilde{\omegab}_m)}. 
\end{align*}
Let us define $C^2\triangleq \E_{dp}[{k(\xb,\xb)}]$. Taking expectation from above with respect to both $d\tau$ (from which $\tilde{\omegab}_m$'s are sampled) and $dp(\xb)$, we have
\begin{align}\label{e1}
   \E[e_1]\leq \frac{4\sqrt{F}GC}{\sqrt{N}}. 
\end{align}
where we applied Jensen's inequality again.
\subsection*{Bounding $e_2$}
To bound $e_2$, we start by Assumption \ref{A1} (G-Lipschitz loss) to get {\small
\begin{align*}
    &R(\fh_{\betab})-R(f_{\betab})\leq G\norm{\fh_{\betab}-f_{\betab}}_{\Lc^2(dp,\Xc)}\\
    &=G\norm{\sum_{m=1}^M\beta_{m}\phi(\cdot,\omegab_m)\left(\frac{1}{\sqrt{\qh(\omegab_m)}}-\frac{1}{\sqrt{q(\omegab_m)}}\right)}_{\Lc^2(dp,\Xc)}\\
    &\leq G\sqrt{\frac{2F}{M}\sum_{m=1}^M\left(\frac{1}{\sqrt{\qh(\omegab_m)}}-\frac{1}{\sqrt{q(\omegab_m)}}\right)^2}, \numberthis \label{E24}
\end{align*}}
where the last line follows by $\norm{\betab}^2\leq \frac{2F}{M}$ and the fact that $\sup_{\xb,\omegab}\abs{\phi(\xb,\omegab)} \leq 1$ (Assumption \ref{A2}). Taking expectation with respect to $\qh(\omegab) d\tauh(\omegab)$, we have by Jensen's inequality that 
\begin{align*}
\E_{\qh d\tauh}&[e_2]\leq G\sqrt{\frac{2F}{M_0}\sum_{m=1}^{M_0}\left(1-\frac{\sqrt{\qh(\tilde{\omegab}_m)}}{\sqrt{q(\tilde{\omegab}_m)}}\right)^2}\\
&=G\sqrt{\frac{2F}{M_0}\sum_{m=1}^{M_0}\frac{1}{q(\tilde{\omegab}_m)}\left(\sqrt{q(\tilde{\omegab}_m)}-\sqrt{\qh(\tilde{\omegab}_m)}\right)^2}\\
&\leq G\sqrt{\frac{2F}{M_0}\sum_{m=1}^{M_0}\frac{1}{q(\tilde{\omegab}_m)}\abs{q(\tilde{\omegab}_m)-\qh(\tilde{\omegab}_m)}}, \numberthis\label{E29}
\end{align*}
where the last line follows by the simple inequality that $\abs{\sqrt{a}-\sqrt{b}}\leq\sqrt{\abs{a-b}}$ for $a,b\geq0$. 

Now, let $\pb_i$ be the standard unit vector in $\R^{M_0}$. Note that from relationship \eqref{E18}, we can conclude that for any $i\in [M_0]$
$$
q(\tilde{\omegab}_i)=\frac{[\tilde{\Gb}]_{ii}}{\tr{\tilde{\Gb}}}=\frac{\pb_i^\top \tilde{\Gb}\pb_i}{\tr{\tilde{\Gb}}}\geq \frac{\sigma_{M_0}(\tilde{\Gb})}{\tr{\tilde{\Gb}}},
$$
which allow us to simplify \eqref{E29} to get
\begin{align*}
  \E_{\qh d\tauh}[e_2]\leq  G\sqrt{\frac{2F\tr{\tilde{\Gb}}}{M_0\sigma_{M_0}(\tilde{\Gb})}\sum_{m=1}^{M_0}\abs{q(\tilde{\omegab}_m)-\qh(\tilde{\omegab}_m)}},
\end{align*}
combined with Corollary \ref{C1} resulting in 
\begin{align*}
\E_{\qh d\tauh}[e_2]&\leq  G\sqrt{\frac{8\Delta F\tr{\tilde{\Gb}}}{M_0\sigma_{M_0}(\tilde{\Gb})} \frac{\sum\limits_{m=1}^{M_0}\frac{\lambda}{\lambda+\sigma_m(\Gb)}}{\sum\limits_{m=1}^{M_0}\frac{\sigma_m(\Gb)}{\lambda+\sigma_m(\Gb)}}}\\
&=G\sqrt{\frac{8\Delta F}{M_0\sigma_{M_0}(\tilde{\Gb})} \sum\limits_{m=1}^{M_0}\frac{\lambda}{\lambda+\sigma_m(\Gb)}}\\
&\leq G\sqrt{\frac{8\Delta F}{\sigma_{M_0}(\tilde{\Gb})} \frac{\lambda}{\lambda+\sigma_{M_0}(\Gb)}}\\
&= G\sqrt{8\Delta F \frac{\lambda}{\sigma_{M_0}(\Gb)}},
\end{align*}
with probability $1-\delta$ (over $N_0$ data samples that are sampled out of $N$ in Algorithm \ref{ALGO1}). Assuming $8\Delta F \frac{\lambda}{\sigma_{M_0}(\Gb)}<1$ and letting $\delta=4\Delta F \frac{\lambda}{\sigma_{M_0}(\Gb)}$, the in expectation bound over data will be easily obtained and we have
$$
\E_{dp}\E_{\qh d\tauh}[e_2] \leq 4G\sqrt{\Delta F \frac{\lambda}{\sigma_{M_0}(\Gb)}}.
$$
Finally, we take expectation with respect to $d\tau$ to get
\begin{align}\label{e2}
\E[e_2] \leq 4G\sqrt{\Delta F \E_{d\tau}\left[\frac{\lambda}{\sigma_{M_0}(\Gb)}\right]},
\end{align}
as long as $N_0 \geq \frac{8}{3}\Delta^{-2}M_0\log\delta^{-1}$ and $M_0=o(N)$.
\subsection*{Bounding $e_3$}
We can use G-Lipschitz continuity and apply the in-expectation version of Proposition \ref{Bach} to get
\begin{align}\label{e3}
    \ex{e_3}\leq G8\sqrt{\lambda}.
\end{align}
We should note that Proposition \ref{Bach} is with respect to the measure $qd\tauh$, while we generate the samples in the algorithm by $\qh d\tauh$. This can cause an additional error in \eqref{e3} which is in the order of the total variation distance  $\sum_{m=1}^{M_0}\abs{q(\tilde{\omegab}_m)-\qh(\tilde{\omegab}_m)}$. However, we can safely disregard this error term as we have already bounded a larger error (in orders) when bounding $e_2$ in equation \eqref{E29}.

\subsection*{Bounding $e_4$}
First, notice the change of feasible set for $\alphab$ from $e_3$ to $e_4$. Since $\abs{k(\cdot,\cdot)}\leq 1$ 
\begin{align*}
\forall \alphab: \norm{\alphab}^2\leq \frac{F}{N} \Rightarrow \sum\limits_{i=1}^N\sum\limits_{j=1}^N\alpha_i\alpha_jk(\xb_i,\xb_j)\leq N\norm{\alphab}^2\leq F,
\end{align*}
the feasible set involved in $e_4$ is a subset of the one in $e_3$, and since we are taking infimum, this can only loosen the bound. Since
$$
R(\gh_{\alphab})-R(g_{\gammab}) \leq G\norm{g_{\gammab}-\gh_{\alphab}}_{\Lc^2(dp,\Xc)},
$$
we have
\begin{align*}
    e_4&\leq G\sup_{\norm{\gammab}^2\leq  \frac{F}{N}}\inf_{\norm{\alphab}^2\leq \frac{F}{N}} \left[\norm{g_{\gammab}-\gh_{\alphab}}_{\Lc^2(dp,\Xc)}\right]\\
   & \leq G\sup_{\norm{\gammab}^2\leq  \frac{F}{N}} \left[\norm{g_{\gammab}-\gh_{\gammab}}_{\Lc^2(dp,\Xc)}\right].
\end{align*}
Taking expectation from above and applying Lemma \ref{L7}, we obtain
\begin{align}\label{e4}
\ex{e_4}\leq G\sqrt{F\E_{dp,dp'}\left[\text{var}_{d\tau}\left(\kh_{M_0}(\xb,\xb')\right)\right]}.
\end{align}
\subsection*{Finishing the proof}
Plugging \eqref{e1}, \eqref{e2}, \eqref{e3}, and \eqref{e4} into \eqref{E3} completes the proof.

\bibliographystyle{IEEEtran}
\bibliography{IEEEabrv,shahin}

\end{document}

%% file: header.tex
\newcommand\numberthis{\addtocounter{equation}{1}\tag{\theequation}}

%%%%% Algorithm

\newcommand{\alg}{\text{ELSS}}

%%%%% BB

\newcommand{\0}{\mathbb{0}}
\newcommand{\1}{\mathbb{1}}

\newcommand{\E}{\mathbb{E}}
\newcommand{\R}{\mathbb{R}}
\renewcommand{\P}{\mathbb{P}}
\newcommand{\U}{\mathbb{U}}

%%%% BF

\newcommand{\ab}{\mathbf{a}}
\newcommand{\eb}{\mathbf{e}}
\newcommand{\ib}{\mathbf{i}}
\newcommand{\pb}{\mathbf{p}}
\newcommand{\qb}{\mathbf{q}}
\newcommand{\xb}{\mathbf{x}}
\newcommand{\yb}{\mathbf{y}}

\newcommand{\Ab}{\mathbf{A}}
\newcommand{\Bb}{\mathbf{B}}
\newcommand{\Cb}{\mathbf{C}}
\newcommand{\Db}{\mathbf{D}}
\newcommand{\Eb}{\mathbf{E}}
\newcommand{\Fb}{\mathbf{F}}
\newcommand{\Gb}{\mathbf{G}}
\newcommand{\Hb}{\mathbf{H}}
\newcommand{\Ib}{\mathbf{I}}
\newcommand{\Jb}{\mathbf{J}}
\newcommand{\Kb}{\mathbf{K}}
\newcommand{\Qb}{\mathbf{Q}}
\newcommand{\Rb}{\mathbf{R}}

\newcommand{\alphab}{\boldsymbol{\alpha}}
\newcommand{\betab}{\boldsymbol{\beta}}
\newcommand{\gammab}{\boldsymbol{\gamma}}
\newcommand{\phib}{\boldsymbol{\phi}}
\newcommand{\Phib}{\boldsymbol{\Phi}}
\newcommand{\Qhib}{\boldsymbol{\Qhi}}
\newcommand{\omegab}{\boldsymbol{\omega}}
\newcommand{\psib}{\boldsymbol{\psi}}
\newcommand{\sigmab}{\boldsymbol{\sigma}}
\newcommand{\nub}{\boldsymbol{\nu}}
\newcommand{\thetab}{\boldsymbol{\theta}}
\newcommand{\delb}{\boldsymbol{\delta}}

%%%%%% CAL

\newcommand{\Cc}{\mathcal{C}}
\newcommand{\Ec}{\mathcal{E}}
\newcommand{\Fc}{\mathcal{F}}
\newcommand{\Hc}{\mathcal{H}}
\newcommand{\Lc}{\mathcal{L}}
\newcommand{\Nc}{\mathcal{N}}
\newcommand{\Oc}{\mathcal{O}}
\newcommand{\Pc}{\mathcal{P}}
\newcommand{\Rc}{\mathcal{R}}
\newcommand{\Uc}{\mathcal{U}}
\newcommand{\Xc}{\mathcal{X}}
\newcommand{\Yc}{\mathcal{Y}}

%%%%%%% Hats

\newcommand{\tauh}{\widehat{\tau}}
\newcommand{\Sigmah}{\widehat{\Sigma}}

\newcommand{\fh}{\widehat{f}}
\newcommand{\gh}{\widehat{g}}
\newcommand{\kh}{\widehat{k}}
\newcommand{\qh}{\widehat{q}}
\newcommand{\Rh}{\widehat{R}}

%%%%%% Bold Hats

\newcommand{\alphabh}{\widehat{\boldsymbol{\alpha}}}
\newcommand{\thetabh}{\widehat{\boldsymbol{\theta}}}

\newcommand{\qbh}{\widehat{\mathbf{q}}}

\newcommand{\Kbh}{\widehat{\mathbf{K}}}

%%%%% Cal Hats

\newcommand{\Fch}{\widehat{\mathcal{F}}}

%%%%%  Optimization Commands

\newcommand{\argmin}{\text{argmin}}
\newcommand{\arginf}{\text{arginf}}
\newcommand{\argmax}{\text{argmax}}
\newcommand{\minimize}{\text{minimize}}
\newcommand{\maximize}{\text{maximize}}
\newcommand{\supp}{\text{supp}}

%%%%%%% Other commands

\newcommand{\TV}{\text{TV}}
\newcommand{\norm}[1]{\left\lVert#1\right\rVert}
\newcommand{\tr}[1]{\text{Tr}\left[#1\right]}
\newcommand{\inn}[1]{\left<#1\right>}
\newcommand{\seal}[1]{\left \lceil #1\right \rceil}
\newcommand{\floor}[1]{\left \lfloor #1\right \rfloor}
\newcommand{\abs}[1]{\left|#1\right|}
\newcommand{\ind}[1]{\mathbf{1}\left(#1\right)}
\newcommand{\ex}[1]{\E\left[#1\right]}

%%%%% Theorems

\newtheorem{theorem}{Theorem}
\newtheorem{acknowledgement}[theorem]{Acknowledgement}
\newtheorem{assumption}{Assumption}
\newtheorem{conjecture}[theorem]{Conjecture}
\newtheorem{corollary}[theorem]{Corollary}
\newtheorem{definition}{Definition}
\newtheorem{example}{Example}
\newtheorem{lemma}[theorem]{Lemma}
\newtheorem{fact}{Fact}
\newtheorem{problem}{Problem}
\newtheorem{proposition}[theorem]{Proposition}
\newtheorem{remark}{Remark}
\newtheorem{solution}[theorem]{Solution}
\newtheorem{summary}[theorem]{Summary}